\documentclass[sigconf]{acmart}

\usepackage{booktabs}
\usepackage{tikz}
\usetikzlibrary{shapes.geometric, arrows,positioning}
\tikzstyle{process} = [rectangle, rounded corners, minimum width=3cm, minimum height=1cm,text centered, draw=black, fill=orange!45]
\tikzstyle{guarantee} = [rectangle, rounded corners, minimum width=3cm, minimum height=1cm,text centered, draw=black, fill=blue!15]
\tikzstyle{transform} = [rectangle, rounded corners, minimum width=3cm, minimum height=1cm,text centered, draw=black, line width=0.5mm, fill=red!30]
\tikzstyle{arrow} = [thick,->,>=stealth]
\setcopyright{none}

\acmConference[]{}
\copyrightyear{2018}

\renewcommand\footnotetextcopyrightpermission[1]{}
\begin{document}
\title{Provably Fair Representations}

\author{Daniel McNamara}
\affiliation{%
  \institution{The Australian National University and CSIRO Data61}
}
\email{daniel.mcnamara@anu.edu.au}

\author{Cheng Soon Ong}
\affiliation{%
  \institution{The Australian National University and CSIRO Data61}
}
\email{chengsoon.ong@anu.edu.au}

\author{Robert C. Williamson}
\affiliation{%
  \institution{The Australian National University and CSIRO Data61}
}
\email{bob.williamson@anu.edu.au}

\begin{abstract}
Machine learning systems are increasingly used to make decisions about people's lives, such as whether to give someone a loan or whether to interview someone for a job. This has led to considerable interest in making such machine learning systems fair. One approach is to transform the input data used by the algorithm. This can be achieved by passing each input data point through a representation function prior to its use in training or testing. Techniques for learning such representation functions from data have been successful empirically, but typically lack theoretical fairness guarantees. We show that it is possible to prove that a representation function is fair according to common measures of both group and individual fairness, as well as useful with respect to a target task. These provable properties can be used in a governance model involving a data producer, a data user and a data regulator, where there is a separation of concerns between fairness and target task utility to ensure transparency and prevent perverse incentives. We formally define the `cost of mistrust' of using this model compared to the setting where there is a single trusted party, and provide bounds on this cost in particular cases. We present a practical approach to learning fair representation functions and apply it to financial and criminal justice datasets. We evaluate the fairness and utility of these representation functions using measures motivated by our theoretical results.
\end{abstract}

\maketitle

\section{Introduction}
\label{introduction}
Machine learning algorithms play an ever-increasing role in society. Such algorithms are widely used to make decisions about individuals, such as whether to grant someone a loan or what advertisements to serve someone online. Other relevant examples are automated applicant screening used by employers and educational institutions, automated assignment grading, and algorithmic advice on criminal sentencing.\footnote{Some other related notions of fairness, such as compliance by search engines with antitrust legislation, may also be readily incorporated within our framework.} A risk of this trend is that such algorithms may be unfair in some way, for example by discriminating against particular groups. Even if not intended by the algorithm designer, discrimination is possible because the reasoning behind the algorithm's decisions is often difficult for humans to interpret. Furthermore, artefacts of previous discrimination present in the algorithm's training data may increase this tendency in the algorithm's decisions. 

To remove or minimize discrimination effects caused by the use of machine learning systems, a fairness objective may be incorporated into algorithm design. This approach benefits users of such algorithms, particularly those in social groups that are potentially the subject of discrimination on grounds such as race or gender. Such an approach also assists companies and organisations deploying machine learning systems in ensuring regulatory compliance. Moreover, as rapid technological progress drives disruptive social change and in turn resistance to such change, a focus on fairness will be required to maintain the `social license to operate' \cite{morrison2014social} of companies using such algorithms.\footnote{For example, the European Union has proposed `right to explanation' laws scheduled to commence in 2018, allowing a user to ask for an explanation of an algorithmic decision made about them. \cite{goodman2016eu}}

We consider the task of producing a \textit{decision variable} (e.g. whether to grant a loan) which predicts a \textit{target variable} of interest (e.g. loan default), while at the same time avoiding discrimination on the basis of an individual's group membership (e.g. race, gender) encoded in a \textit{sensitive variable}. We adopt the approach of passing input data about individuals through a representation function, so that subsequent use of the cleaned data will not be able to discriminate based on the sensitive variable. By incorporating fairness as a data pre-processing step, we achieve a separation of concerns between fairness and target task utility which offers governance and regulatory advantages.

Our work provides theoretical guarantees for several performance measures of interest:
\begin{enumerate}
\item \textbf{Group fairness}: Using common measures of the similarity of decisions for one group compared to another, we show that it is possible for a representation function to improve group fairness.
\item \textbf{Individual fairness}: We quantify the effect the representation function has on the assignment of similar decisions to individuals who are similar.
\item \textbf{Target utility}: We quantify the cost a representation function incurs in terms of target variable utility.
\item \textbf{Cost of mistrust}: We quantify the cost of this separation of concerns, relative to a setting where the decision-making party is given full access to the sensitive variable. 
\end{enumerate}

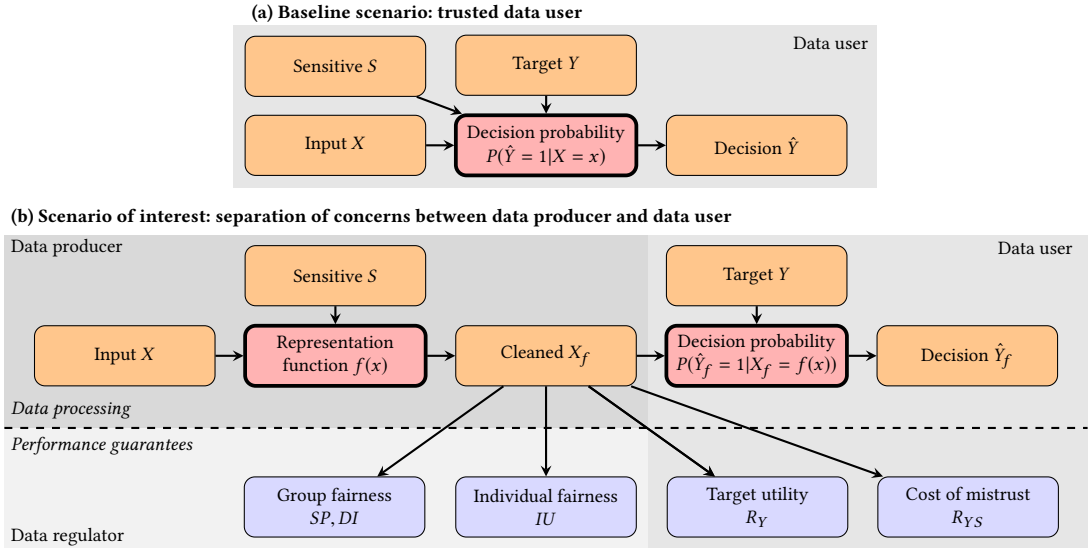
\begin{figure*}

\centering
\begin{tikzpicture}[node distance=3.5cm,scale=0.8,transform shape]

\draw[fill=gray!30,draw=none] (-2,-1.2) rectangle (8.7,2);
\draw[fill=gray!20,draw=none] (8.7,-3.2) rectangle (16,2);
\draw[fill=gray!10,draw=none] (-2,-3.2) rectangle (8.7,-1.2);
\node[align=center] (data) [process] {Input $X$};
\node[align=center] (f) [transform,right of = data] {Representation \\ function $f(x)$};
\node[align=center] (sensitive) [process,above = 0.3cm of f] {Sensitive $S$};
\node (cleaned) [process,right of = f] {Cleaned $X_f$};
\node[align=center] (probability) [transform,right of = cleaned] {Decision probability \\ $P(\hat{Y}_f=1|X_f=f(x))$};
\node (predictions) [process,right of = probability] {Decision $\hat{Y}_f$};
\node (labels) [process,above = 0.3 cm of probability] {Target $Y$};
\node [align=center](group_fairness) [guarantee,below left = 1.5cm and 0.5cm of cleaned] {Group fairness \\ $SP, DI$};
\draw [arrow] (data) -- node[anchor=south,align=center, text width=2cm] {} (f);
\draw [arrow] (sensitive) -- node[anchor=south,align=center, text width=2cm] {} (f);
\draw [arrow] (f) -- node[anchor=south,align=center, text width=2cm] {} (cleaned);
\draw [arrow] (cleaned) -- node[anchor=north east, shift={(-0.7,-0.1)}] {} (group_fairness);
\draw [arrow] (cleaned) -- node[anchor=south,align=center, text width=3cm] {} (probability);
\draw [arrow] (probability) -- node[anchor=south,align=center, text width=3cm] {} (predictions);
\node[align=center] (individual_fairness) [guarantee,below=1.5cm of cleaned] {Individual fairness \\ $IU$};
\node[align=center] (accuracy) [guarantee,below right = 1.5cm and 0.5cm of cleaned] {Target utility \\ $R_Y$};
\node[align=center] (combined) [guarantee,below right = 1.5cm and 4cm of cleaned] {Cost of mistrust \\ $R_{YS}$};
\draw [arrow] (cleaned) -- node[anchor=north west,shift={(0.7,-0.1)}] {} (accuracy);
\draw [arrow] (cleaned) -- node[anchor=north east,shift={(0.1,-0.1)}] {} (individual_fairness);
\draw [arrow] (cleaned) -- node[anchor=west] {} (accuracy);
\draw [arrow] (cleaned) -- node[anchor=north west,shift={(1,-0.1)}] {} (combined);
\draw [arrow] (labels) -- node[anchor=south] {} (probability);
\draw [dashed,thick] (-2,-1.2) -- (16,-1.2);
\node[align=left,anchor=west] at (-2,-0.9) {\textit{Data processing}};
\node[align=left,anchor=west] at (-2,-1.5) {\textit{Performance guarantees}};
\node[align=left,anchor=west] at (-2,1.8) {Data producer};
\node[align=left,anchor=west] at (14.4,1.8) {Data user};
\node[align=left,anchor=west] at (-2,-3) {Data regulator};
\node[align=left,anchor=west] at (-2,2.3) {\textbf{(b) Scenario of interest: separation of concerns between data producer and data user}};
\node[align=left,anchor=west] at (2,5.7) {\textbf{(a) Baseline scenario: trusted data user}};
\draw[fill=gray!20,draw=none] (1.8,2.8) rectangle (12.5,5.5);
\node[align=left,anchor=west] at (11,5.2) {Data user};
\node (predictions_original) [process,above of = probability] {Decision $\hat{Y}$};
\node[align=center] (probability_original) [transform,left of = predictions_original] {Decision probability \\ $P(\hat{Y}=1|X=x)$};
\node (labels_original) [process,above = 0.3 cm of probability_original] {Target $Y$};
\node (input_original) [process,above of = f] {Input $X$};
\node (sensitive_original) [process,above = 0.3cm of input_original] {Sensitive $S$};
\draw [arrow] (labels_original) -- node[anchor=south] {} (probability_original);
\draw [arrow] (input_original) -- node[anchor=south,align=center, text width=3cm] {} (probability_original);
\draw [arrow] (probability_original) -- node[anchor=south,align=center, text width=3cm] {} (predictions_original);
\draw [arrow] (sensitive_original) -- node[anchor=north west,shift={(1,-0.1)},align=center, text width=2cm] {} (probability_original);
\end{tikzpicture}

\caption{Problem setting for provably fair representations. We show train time procedures for (a) a baseline scenario when the data user is trusted with access to the sensitive variable and (b) the scenario of interest when the data user is only able to access cleaned data prepared by the data producer. Test time procedures for both scenarios are discussed in Section \ref{introduction}. The performance guarantees achieved in scenario (b) are described in Section \ref{theory}.}
\label{conceptual_fig}
\end{figure*}

 Previous work on fairness in machine learning has mostly considered the case where there is a single trusted \textit{data user} \cite{menon2017cost,zafar2017fairness,edwards_censoring_2015,zemel_learning_2013}. This baseline scenario is shown in Figure \ref{conceptual_fig}(a). The data user learns decision variable $\hat{Y}$ by training on samples of input variable $X$, target variable $Y$ and sensitive variable $S$. We assume that $S$, $Y$ and $\hat{Y}$ are binary variables, a restriction which still allows us to work on many problems of interest, as we show. $X$ may be discrete or continuous. At test time, on input $X=x$, the data user makes decision $P(\hat{Y}=1|X=x)$. We assume that $S$ is not available at test time. However, this is not overly restrictive since $X$ may be constructed to perfectly predict $S$ (e.g. $S$ is one of the attributes included in $X$).

In our scenario of interest, shown in Figure \ref{conceptual_fig}(b), the data user is treated as \textit{untrusted}. This may be appropriate in a setting where the data user's incentives are aligned to target task utility, possibly at the expense of fairness. Our scenario of interest involves three parties: a \textit{data producer} who prepares the input data, a \textit{data user} who makes decisions from the data, and a \textit{data regulator} who oversees fair use of the data. For example, in the context of deciding whether to give an individual a loan, the data producer might be a credit bureau, the data user a bank and the data regulator a government authority. Even if the data producer and data user (and potentially even the data regulator) are the same organization, this conceptual framework provides improved transparency.

We consider the case where the data user only receives access to cleaned data prepared by the data producer, in order to ensure the fairness of the decisions made. The data producer passes the input variable $X$ through a representation function $f$ to produce cleaned variable $X_f$, which is then made available to the data user who learns decision variable $\hat{Y}_f$ using $X_f$ and $Y$.  We assume $\hat{Y}_f$ is also binary. We assume $X_f$ has the same domain as $X$, which means that the cleaned data is interpretable using existing domain knowledge. The representation function $f$ is learned by the data producer from samples of the input variable $X$ and sensitive variable $S$. At test time, on input $X=x$, the data producer computes $X_f=f(x)$, after which the data user makes decision $P(\hat{Y}_f=1|X_f \nobreak = \nobreak f(x))$.

The data producer computes performance guarantees (as described above) about \textit{group fairness} and \textit{individual fairness} to the data regulator and about \textit{target task utility} and the \textit{cost of mistrust} to the data user. In practice, the statistics required for the guarantees must be estimated from a finite data sample.

It is not necessary for the data producer to have access to target variable labels $Y$ which may be used by the data user for supervised learning. This is a realistic situation in the case where the labels are constructed using proprietary transactions (e.g. previous loans) conducted by the data user. Furthermore, assuming that $Y$ is not modified ensures that the data user still has an incentive to accurately predict these labels. The properties of $f$ we guarantee hold for any target task, despite the fact that $Y$ may not be accessible when $f$ is learned.

The remainder of the paper is structured as follows. In Section \ref{background} we introduce related work. In Section \ref{theory} we present our theoretical results. In Section \ref{practical} we propose a practical algorithm for learning fair representation functions, apply it to real-world financial and criminal justice datasets, and analyze the key performance measures suggested by our theoretical results. We conclude and present ideas for future work in Section \ref{conclusion}. We defer proofs of our theorems to Section \ref{proofs}.

\section{Background}
\label{background}
Fairness in machine learning has emerged as a significant field of research in recent years \cite{datta_automated_2015,hardt_equality_2016,feldman2015certifying,zliobaite_survey_2015}. A seminal study providing theoretical foundations for fairness in machine learning \cite{dwork_fairness_2012} proposed two notions of fairness. \textit{Group fairness} can be defined, roughly, as similar (or at least more similar than a baseline\footnote{For example, an affirmative action system may introduce quotas which do not aim for `equality of outcome' but rather `reduced inequality of outcome'.}) decisions for one group compared to another. This is the type of fairness we focus on improving. The \textit{individual fairness} of a machine learning algorithm can be defined, roughly, as the assignment of similar decisions to individuals who are similar by some metric. The two notions of fairness are potentially in tension with each other: group fairness achieves \textit{equal outcomes} for each group regardless of the characteristics of the individuals that make up the groups, while individual fairness provides individuals who are similar with \textit{equal treatment} regardless of their group membership. We are interested in the problem of improving group fairness, without too great a cost to individual fairness or to performance on the target task of interest.

A common approach to measuring group fairness is to compare the decision variable probabilities conditioned on different values of the sensitive variable. The difference in these probabilities was proposed as \textit{statistical parity}\cite{dwork_fairness_2012} and is commonly in use \cite{edwards_censoring_2015,menon2017cost}. The ratio of these probabilities, known as \textit{disparate impact}, is favoured in certain applications such as the 80 percent rule advocated by the US Equal Employment Opportunity Commission \cite{eeoc} and has also been studied in the literature \cite{feldman2015certifying,menon2017cost}. It may be appropriate to require group fairness only on some subset of the input data \cite{calmon2017optimized}, such as requiring equal employment outcomes for men and women with suitable qualifications rather than for all men and women. Disparate mistreatment, which compares the proportion of the incorrect decisions for different values of the sensitive variable, was proposed by \cite{hardt_equality_2016} in the context of a critique of previous notions of group fairness. However, as \cite{zafar2017fairness} points out, when historical decisions are biased, measures such as statistical parity and disparate impact which are independent of these decisions may be more suitable. We focus on these two measures of group fairness but expect that the appropriate choice of fairness measure will be problem-dependent.

There are several approaches to improving the group fairness of machine learning algorithms. For example, it is possible to embed a fairness criterion in the supervised learning objective function \cite{menon2017cost,zafar2017fairness} or post-process the outputs of a classifier \cite{hardt_equality_2016}. We adopt the approach of passing input data about individuals through a representation function, so that subsequent use of the cleaned data will enforce group fairness \cite{zemel_learning_2013}. This approach allows the processes of achieving fairness and making accurate predictions to be separated. Moving responsibility for fairness to a third party may provide better alignment of roles and incentives to ensure fairness is given sufficient attention in contexts where target task utility is the primary focus. The pre-processing approach is common in privacy-preserving data mining, where the data is first `sanitized' before any attempt to make predictions from it \cite{aggarwal2008general}. 

Naively, we might expect that removing the sensitive variable describing an individual's group membership is sufficient. However, the problem of `redundant encodings' means that the sensitive variable may be readily predicted from other features in the data. A classic example is the practice of `redlining', where the neighbourhood in which loan applicants live is used as a proxy for discrimination on the basis of race or class \cite{dwork_fairness_2012}. A solution is to pass the input data through a representation function, which is learned by encoding a desirable fairness property in an objective function which is maximised over training data \cite{zemel_learning_2013}. Modifying both the inputs and the target labels has also been proposed \cite{calmon2017optimized}. However, modifying the target label is problematic since there may be no incentive for a data user interested in target task utility to attempt to accurately predict the target label if it has been modified. We therefore focus on modifying the input data only.

The first technique proposed for fair representations \cite{zemel_learning_2013} learns a transformation of each input data point into a distribution over cluster centers. A target task linear classifier using the cleaned data is learned simultaneously, using an objective function containing terms for each of individual fairness, group fairness and target task utility. Subsequently the `adversarial' approach \cite{edwards_censoring_2015,beutel2017data} to learning fair representations was developed. This uses a neural network to learn a representation function such that an adversary network predicting the sensitive variable from the cleaned data has low accuracy, while another network predicting the target variable using the cleaned data has high accuracy. Other works have proposed learning a representation function such that using the cleaned data, conditioning on different sensitive variable values yields similar distributions \cite{louizos_variational_2015,feldman2015certifying,johndrow2017algorithm}. For example the `variational fair autoencoder' \cite{louizos_variational_2015} matches the means of the distributions. A drawback of previous empirical approaches to learning fair representations is that they provide no theoretical guarantees that information about the sensitive variable in the cleaned data cannot `leak' into a subsequent decision variable, relying only on the fact that \textit{some particular learner} failed to be unfair using the cleaned data.  In fact, a data user focused only on target task utility may have a \textit{perverse incentive} to use a weak learner which makes the representation function appear fair. This problem highlights the need for provably fair representation functions. 

\section{Theoretical guarantees for fair and useful representations}
\label{theory}
We propose theoretical guarantees for learning representation functions which improve group fairness, while preserving individual fairness and target task utility.

\subsection{Guaranteeing group fairness}
\label{group_fairness}
We would like to guarantee that a representation function $f$ achieves group fairness if used to transform $X$ to $X_f$. To formalize group fairness, we first consider statistical parity \cite{dwork_fairness_2012}, which measures the difference in outcomes for different sensitive variable groups.
\begin{definition}[Statistical parity]
Decision variable $\hat{Y}$ and sensitive variable $S$ have statistical parity 
\begin{equation*}
SP(\hat{Y},S):=p(\hat{Y}=1|S=1)-p(\hat{Y}=1|S=0).
\end{equation*}
\end{definition}

We also consider the related notion of normalized disparate impact, which measures the ratio of outcomes for different sensitive variable groups.\footnote{Disparate impact defined as $\frac{p(\hat{Y}=1|S=0)}{p(\hat{Y}=1|S=1)}$ is commonly used, although we prefer normalized disparate impact so that like statistical parity, smaller values are fairer.} 

\begin{definition}[Normalized disparate impact]
Decision variable $\hat{Y}$ and sensitive variable $S$ have normalized disparate impact
\begin{equation*}
DI(\hat{Y},S):=1-\frac{p(\hat{Y}=1|S=0)}{p(\hat{Y}=1|S=1)}.
\end{equation*}
\end{definition}

We assume $p(\hat{Y}=1|S=1)\geq p(\hat{Y}=1|S=0)$ without loss of generality, since the assignment of sensitive variable groups to the values $0$ and $1$ is arbitrary. Therefore $SP(\hat{Y},S), DI(\hat{Y},S) \in [0,1]$ where $0$ is perfectly fair and $1$ is maximally unfair.

\subsubsection{Statistical parity guarantee}

The statistical parity of a decision variable is closely related to the balanced error rate with respect to the sensitive variable, as observed in \cite{menon2017cost}.

\begin{definition}[Balanced error rate]
Decision variable $\hat{Y}$ has balanced error rate with respect to sensitive variable $S$ defined as
\begin{equation*}
BER(\hat{Y},S):=\frac{1}{2}p(\hat{Y}=1|S=0)+\frac{1}{2}p(\hat{Y}=0|S=1).
\end{equation*}
\end{definition}

While in our problem setting the role of the decision variable is to predict the target variable $Y$ rather than $S$, an adversary may instead make decisions using predictions of $S$. We may tightly upper bound the statistical parity of any decision made using $X_f$ using the minimum achievable balanced error rate, as shown in Theorem \ref{group_fairness_theorem_simplified}.  The relationship between $BER$ and $SP$ for a \textit{particular} classifier was previously observed by \cite{menon2017cost}, although our result provides an $SP$ guarantee for \textit{any} classifier using $X_f$.

\begin{theorem}[Tight bound on statistical parity]
\label{group_fairness_theorem_simplified}

For all decisions $\hat{Y}_f$ made using $X_f$ as input, 
\begin{equation*}
SP(\hat{Y}_f,S)\leq 1-2BER(\hat{Y}_f^{BER},S).
\end{equation*}
The bound is tight when $\hat{Y}_f=\hat{Y}_f^{BER}$, where 
\begin{equation*}
\hat{Y}_f^{BER}(x):=\mathbf{1}(p(S=1|X_f=x)\geq p(S=1)).
\end{equation*}
\end{theorem}

\subsubsection{Alternative statistical parity guarantee using conditional entropy}

Applying Theorem \ref{group_fairness_theorem_simplified} involves computing the decision variable which minimizes balanced error rate when predicting $S$ from $X_f$. This involves making decisions by thresholding the distribution $p(S=1|X_f=x)$. Because of the discontinuity introduced by this thresholding function, the method may be sensitive to errors in the estimation of $p(S=1|X_f=x)$ and $P(S=1)$ from a finite sample. We present an alternative upper bound on statistical parity, which avoids both the thresholding function and the need to estimate $P(S=1)$ by instead estimating the conditional entropy of $S$ given $X_f$.

We introduce two definitions from information theory which we will require for our statistical parity guarantee.

\begin{definition}[Binary entropy]
For a binary variable with success probability $p$, the binary entropy is 
\begin{equation*}
H_b(p):=-p\log_2 p - (1-p)\log_2(1-p).
\end{equation*}
\end{definition}

\begin{definition}[Conditional entropy]
For a binary variable $S$ and input variable $X$, the conditional entropy
\begin{equation*}
H(S|X):=\mathbb{E}_x[H_b(p(S=1|X=x))].
\end{equation*}
\end{definition}

In Theorem \ref{group_fairness_theorem} we upper bound statistical parity using $H(S|X_f)$. Hence, making $H(S|X_f)$ large ensures that \textit{any} decision-making procedure using $X_f$ will have small statistical parity.

\begin{theorem}[Bound on statistical parity using conditional entropy]
\label{group_fairness_theorem}
Let $H_b^{-1}$ be the inverse of $H_b(p)$ where $p \in [0,\frac{1}{2}]$. For all decisions $\hat{Y}_f$ made using $X_f$ as input, 
\begin{equation*}
SP(\hat{Y}_f,S)\leq 1-\frac{H_b^{-1}(H(S|X_f))}{\max(p(S=1),p(S=0))}.
\end{equation*}
\end{theorem}

\subsubsection{Normalized disparate impact guarantee}

We may tightly upper bound the normalized disparate impact of any decision made using $X_f$, as shown in Theorem \ref{DI_minimiser}.  

\begin{theorem}[Tight bound on normalized disparate impact]
\label{DI_minimiser}
Let $\eta_f:=\underset{x}{\max\text{ }}p(S=1|X_f=x)$. Let
\begin{equation*}
\hat{Y}_f^{DI}(x):=
\begin{cases} 
      \gamma & \text{if $p(S=1|X_f=x)=\eta_f$,} \\
      0 & \text{otherwise.}
   \end{cases}
\end{equation*}
where $\gamma$ is an arbitrary positive constant. For all decisions $\hat{Y}_f$ made using $X_f$ as input, 
\begin{equation*}
DI(\hat{Y}_f,S)\leq 1-\frac{p(S=1)(1-\eta_f)}{p(S=0)\eta_f}.
\end{equation*}
The bound is tight when $\hat{Y}_f=\hat{Y}_f^{DI}$.
\end{theorem}

Previous work has provided upper bounds on $DI$ for a \textit{particular} decision using balanced error rate and the related notion of cost sensitive risk \cite{feldman2015certifying,menon2017cost}. However, our result applies to \textit{any} decision made using $X_f$ as input. Furthermore, $p(\hat{Y}_f^{DI}=0|S=1)\approx 1$, $p(\hat{Y}_f^{DI}=1|S=0)\approx 0$ and so $BER(\hat{Y}_f^{DI},S)\approx \frac{1}{2}$ regardless of $X_f$, while if $\eta_f \approx 1$ then $DI(\hat{Y}_f^{DI},S)\approx 1$. Hence ensuring $BER(\hat{Y}_f,S)$ is large for all $\hat{Y}_f$ is not sufficient to guarantee that $DI(\hat{Y}_f,S)$ is small for all $\hat{Y}_f$. However, we \textit{can} guarantee $DI(\hat{Y}_f)$ will be small for all $\hat{Y}_f$ if $\eta_f$ is not too large. Furthermore we observe that because the proof of Theorem \ref{group_fairness_theorem} relies on the particular relationship of $SP$ and $BER$, there is no immediate analog to this theorem in the case of $DI$.

\subsection{Guaranteeing individual fairness}
\label{individual_fairness}

We introduce the notion of individual fairness from \cite{dwork_fairness_2012}, which roughly states that if two individuals are similar then a similar decision should be made about them. Another term used for this type of (un)fairness is disparate treatment \cite{zafar2017fairness}. This is familiar mathematically as a smoothness property.

\begin{definition}[Individual fairness \cite{dwork_fairness_2012}]
\label{individual_fairness_definition}
Decision variable $\hat{Y}$ is $D,d$ individually fair with respect to input variable $X$ iff $\forall x,x'$, $D(p(\hat{Y}=1|X=x),p(\hat{Y}=1|X=x'))\leq d(x,x')$, where $D$ and $d$ are subadditive functions.
\end{definition}

We also give a quantitative notion of individual \textit{unfairness} by measuring the probability that a pair of randomly selected individuals will be treated fairly.

\begin{definition}[Individual unfairness]
Decision variable $\hat{Y}$ and input variable $X$ have $D,d$ individual unfairness
\begin{equation*}
IU_{D,d}(\hat{Y},X):=p[D(p(\hat{Y}=1|X=x),p(\hat{Y}=1|X=x'))> d(x,x')]
\end{equation*}
where $x$ and $x'$ are independent random samples of $X$.
\end{definition}

We would like to understand the effect on individual fairness of moving from $X$ to $X_f$. We assume the motivation for using $X_f$ is group rather than individual fairness. However, we would like to ensure that group fairness is not achieved at too great a cost to individual fairness. To do this we introduce the following definition.

\begin{definition}[Large reconstruction error rate]
Let $\epsilon$ be a constant and $x$ be a random sample of $X$. Let $p(d(x,f(x))> \epsilon)$ be the large reconstruction error rate.
\end{definition}

In Theorem \ref{individual_fairness_theorem} we provide a result guaranteeing that if $X_f$ approximately reconstructs $X$ for most individuals, as measured by the large reconstruction error rate $p(d(x,f(x))> \epsilon)$, then a decision variable that is individually fair using $X$ can be used to construct a decision variable that is not too individually unfair using $X_f$.

\begin{theorem}[Guarantee for individual fairness]
\label{individual_fairness_theorem}
Suppose decision variable $\hat{Y}$ is individually fair with respect to $X$ and the large reconstruction error rate $p(d(x,f(x))> \epsilon)\leq \delta$. Let $\hat{Y}_f$ be the random variable formed by setting $p(\hat{Y}_f=1|X_f=x)=p(\hat{Y}=1|X=x)$. Let $d_\epsilon(x,x'):=d(x,x')+2\epsilon$. Then $IU_{D,d_\epsilon}(\hat{Y}_f,X)\leq 2\delta$.
\end{theorem}

\subsection{Guaranteeing target task utility}
\label{low_regret}
We would also like to guarantee that using $X_f$ does not have too great an impact on target task utility compared to using $X$. We first define a measure of target task utility, known as risk. 

\begin{definition}[Target variable risk]
\label{original_risk_definition}
Decision variable $\hat{Y}$ has risk with respect to target variable $Y$ defined as 
\begin{equation*}
R_Y(\hat{Y}):=\mathbb{E}_x[D(p(\hat{Y}=1|X=x),p(Y=1|X=x))]
\end{equation*}

where $D$ is a subadditive function.
\end{definition}

We would like to quantify the effect on risk of using a decision variable $\hat{Y}_f$ using $X_f$ as input rather than decision variable $\hat{Y}$ using $X$ as input. To do this we introduce the following definition.

\begin{definition}[Average reconstruction error]
Let the average reconstruction error incurred by $f$ for input variable $X$ be 
\begin{equation*}
\mathbb{E}_x[d(x,f(x))].
\end{equation*}
\end{definition}

We would like to guarantee that if $X_f$ approximately reconstructs $X$,  measured using the average reconstruction error $\mathbb{E}_x[d(x,f(x))]$, then $R_Y(\hat{Y}_f)$ is not too much larger than $R_Y(\hat{Y})$. We show this in Theorem \ref{accuracy_theorem}. The theorem requires a smoothness property of $\hat{Y}$, which is equivalent to our previous definition of individual fairness. Our result applies to any such $\hat{Y}$, so that we do not need to know $\hat{Y}$ in order to provide the guarantee.

\begin{theorem}[Guarantee for target task utility]
\label{accuracy_theorem}
Suppose decision variable $\hat{Y}$ is individually fair with respect to $X$ and has risk $R_Y(\hat{Y})$ with respect to target variable $Y$. Let $\hat{Y}_f$ be the decision variable formed by setting $p(\hat{Y}_f=1|X_f=x)=p(\hat{Y}=1|X=x)$.
 Then 
\begin{equation*}
R_Y(\hat{Y}_f)\leq R_Y(\hat{Y})+\mathbb{E}_x[d(x,f(x))].
\end{equation*}
\end{theorem}

\subsection{Guaranteeing the cost of mistrust}
\label{combined_guarantee}

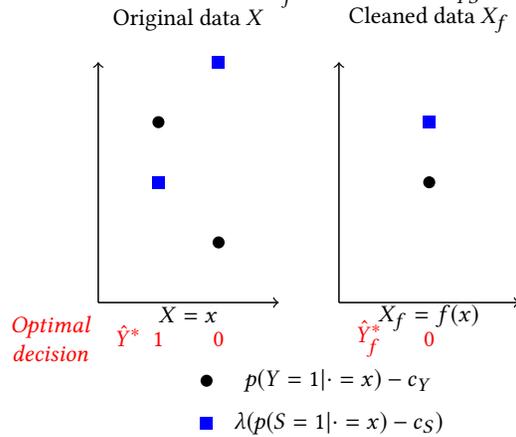
\begin{figure}

\centering
\begin{tikzpicture}[scale=0.8]

\node at (1.5,4.5){Original data $X$};
\node at (1.5,-0.2){$X=x$};
\path [draw=none,fill=black] (1,3) circle (0.1);
\path [draw=none,fill=blue] (0.9,1.9) rectangle (1.1,2.1);
\path [draw=none,fill=black] (2,4) circle (0.1);
\path [draw=none,fill=blue] (1.9,0.9) rectangle (2.1,1.1);
\draw [->,line width=0.2mm] (0,0) -- (3,0);
\draw [->,line width=0.2mm] (0,0) -- (0,4);
\node[red,text width=1.5cm, align=center] at (-0.8,-0.6){\textit{Optimal decision}};
\node[red] at (0.5,-0.55){$\hat{Y}^*$};
\node[red] at (1,-0.6){$1$};
\node[red] at (2,-0.6){$1$};
\node at (-1.2,5)[align=left,anchor=west]{(a) No cost: $R_{YS}(\hat{Y}^*_f)-R_{YS}(\hat{Y}^*)=0$};
\draw [-,dashed,gray,line width=0.2mm] (-1.2,-1.2) -- (7.4,-1.2);

\node at (5.5,-0.25){$X_f=f(x)$};
\path [draw=none,fill=black] (5.5,3.5) circle (0.1);
\path [draw=none,fill=blue] (5.4,1.4) rectangle (5.6,1.6);
\draw [->,line width=0.2mm] (4,0) -- (7,0);
\draw [->,line width=0.2mm] (4,0) -- (4,4);
\node[red] at (4.5,-0.6){$\hat{Y}^*_f$};
\node[red] at (5.5,-0.6){$1$};

\node at (5.5,4.5){Cleaned data $X_f$};

\node at (1.5,-2.3){Original data $X$};
\node at (5.5,-2.3){Cleaned data $X_f$};
\node at (1.5,-7.2){$X=x$};
\path [draw=none,fill=black] (1,-4) circle (0.1);
\path [draw=none,fill=blue] (0.9,-5.1) rectangle (1.1,-4.9);
\path [draw=none,fill=black] (2,-6) circle (0.1);
\path [draw=none,fill=blue] (1.9,-3.1) rectangle (2.1,-2.9);
\draw [->,line width=0.2mm] (0,-7) -- (3,-7);
\draw [->,line width=0.2mm] (0,-7) -- (0,-3);
\node[red] at (1,-7.6){$1$};
\node[red] at (2,-7.6){$0$};
\node[red] at (0.5,-7.55){$\hat{Y}^*$};
\node[red,text width=1.5cm, align=center] at (-0.8,-7.6){\textit{Optimal decision}};
\node at (5.5,-7.25){$X_f=f(x)$};
\path [draw=none,fill=black] (5.5,-5) circle (0.1);
\path [draw=none,fill=blue] (5.4,-4.1) rectangle (5.6,-3.9);
\draw [->,line width=0.2mm] (4,-7) -- (7,-7);
\draw [->,line width=0.2mm] (4,-7) -- (4,-3);
\node[red] at (5.5,-7.6){$0$};
\node[red] at (4.5,-7.6){$\hat{Y}^*_f$};

\path [draw=none,fill=black] (1.8,-8.3) circle (0.1);
\node at (4,-8.3){$p(Y=1|\cdot=x)-c_Y$};
\path [draw=none,fill=blue] (1.7,-9.1) rectangle (1.9,-8.9);
\node at (4,-9){$\lambda(p(S=1|\cdot=x)-c_S)$};
\node at (-1.2,-1.8)[align=left,anchor=west]{(b) Maximum cost: $R_{YS}(\hat{Y}^*_f)-R_{YS}(\hat{Y}^*)=R_{YS}^{\max}-R_{YS}(\hat{Y}^*)$};

\end{tikzpicture}

\caption{Illustration of $R_{YS}(\hat{Y}^*_f)-R_{YS}(\hat{Y}^*)$, the \textit{cost of mistrust}. Assume the distribution over $X$ is equally concentrated on two points, and that $f$ maps all points to the same arbitrary value of $X_f$. The vertical axis shows scaled conditional probabilities of $Y$ and $S$ using inputs $X$ and $X_f$. The values of the optimal decisions minimizing $R_{YS}$, $\hat{Y}^*$ on input $X$ and $\hat{Y}_f^*$ on input $X_f$, are shown below each plot. $R_{YS}$ is the combined risk described in Definition \ref{R_definition} and $R_{YS}^{\max}$ is defined in Lemma \ref{R_minimizer}, where we have $R_{YS}(\hat{Y}_f^*)\leq R_{YS}^{\max}$. For a particular value of $R_{YS}(\hat{Y}^*)$ the cost of mistrust is zero in example (a) and maximized in example (b).}
\label{mistrust_fig}
\end{figure}

\begin{table*}[]
\centering
\caption{Characteristics of the Adult and ProPublica datasets}
\label{characteristics}
\begin{tabular}{lp{2.5cm}p{2.5cm}cccccc}
\hline
 Dataset          & $Y=1$ definition              & $S=1$ definition             & $p(Y=1)$ & $p(S=1)$ & $p(Y=1|S=1)$ & $p(Y=1|S=0)$ & $SP(Y,S)$ & $DI(Y,S)$ \\
\hline
Adult      & Income\textgreater\$50,0000 & Gender male                & 0.243  & 0.671  & 0.308      & 0.111      & 0.197   & 0.639   \\
ProPublica & Reoffended within two years & Ethnicity African-American & 0.511  & 0.484  & 0.587      & 0.439      & 0.148   & 0.252  \\
\hline
\end{tabular}
\end{table*}

The problem of making accurate predictions of the target variable with a group fairness penalty has previously been posed as minimizing a risk $R_{YS}$ defined as the weighted difference of cost-sensitive risks for the target and sensitive variables \cite{menon2017cost}. 

\begin{definition}[Combined risk $R_{YS}$ for group fairness and target utility \cite{menon2017cost}]
\label{R_definition}
For a decision variable $\hat{Y}$, a target variable $Y$ and a sensitive variable $S$, let target variable risk have the form $R_{Y}(\hat{Y}):=$
\begin{equation*}
c_YP(Y=0)p(\hat{Y}=1|Y=0)+(1-c_Y)p(Y=1)p(\hat{Y}=0|Y=1)
\end{equation*}
and sensitive variable risk $R_{S}(\hat{Y}):=$
\begin{equation*}
c_SP(S=0)p(\hat{Y}=1|S=0)+(1-c_S)p(S=1)p(\hat{Y}=0|S=1)
\end{equation*}
 where $c_Y,c_S \in [0,1]$. Let 
\begin{equation*}
R_{YS}(\hat{Y}):=R_{Y}(\hat{Y})-\lambda R_{S}(\hat{Y})
\end{equation*}
where $\lambda$ is a non-negative constant. 
\end{definition}

We observe that $R_{YS}$ may be negative. Both $R_Y$ and $R_S$ are cost-sensitive risks. This form of $R_Y$, which differs from Definition \ref{original_risk_definition}, will be used for the remainder of the paper. Our choices of the parameters $c_Y$ and $c_S$ affect the evaluation of the decision variable. We saw in Theorem \ref{group_fairness_theorem_simplified} that balanced error rate, which is a special case of $R_S$, is closely related to statistical parity. Furthermore, \cite{menon2017cost} showed that $R_{S}$ is also related to disparate impact. 

The decision variable minimizing $R_{YS}$ was found in \cite{menon2017cost}. 

\begin{definition}[Minimizer of $R_{YS}$ \cite{menon2017cost}]
Let $\hat{Y}^*$ be the minimizer of $R_{YS}$ for all decisions made using $X$ as input and let $\hat{Y}_f^*$ be the minimizer of $R_{YS}$ for all decisions made using $X_f$ as input.
\end{definition}

\begin{lemma}[Analytic form of minimizer of $R_{YS}$ \cite{menon2017cost}]
\label{R_minimizer}
\begin{equation*}
\hat{Y}^*(x)=\mathbf{1}(p(Y=1|X=x)-c_Y>\lambda(p(S=1|X=x)-c_S)).
\end{equation*}
Furthermore $R_{YS}(\hat{Y}^*)=$
\begin{equation*}
\mathbb{E}_x[(c_Y-p(Y=1|X=x)-\lambda(c_S-p(S=1|X=x)))\hat{Y}^*(x)]+R_{YS}^{\max}
\end{equation*}
where 
\begin{equation*}
R_{YS}^{\max}:=(1-c_Y)p(Y=1)-\lambda(1-c_S)p(S=1).
\end{equation*}
\end{lemma}

We showed that organizational governance reasons motivate a separation of concerns between a data producer who ensures fairness by producing cleaned data, and a data user learns the target task from the cleaned data. Compared to the minimizer of $R_{YS}$ computed in \cite{menon2017cost} which may be applied in the case of a single trusted party, we incur some excess risk which we refer to as the \textit{cost of mistrust}. 

\begin{definition}[Cost of mistrust]
\label{mistrust_definition}
Let the cost of mistrust be defined as $R_{YS}(\hat{Y}_f^*)-R_{YS}(\hat{Y}^*)$.
\end{definition}

We would like to quantify what this cost of mistrust is. The form of $R_{YS}$, along with the fact that the minimizer of $R_{YS}$ is a thresholding function and is therefore not individually fair as per Definition \ref{individual_fairness_definition}, means that we cannot apply Theorem \ref{accuracy_theorem}. We provide a different guarantee in Theorem \ref{mistrust_theorem} using smoothness properties of the target and sensitive variable distributions conditioned on the input, and the form of the minimizer of $R_{YS}$.

\begin{theorem}[Guarantee for cost of mistrust]
\label{mistrust_theorem}
Suppose for some subadditive function $d$ we have the Lipschitz conditions $\forall x,x'$,
\begin{equation*}
|p(Y=1|X=x)-p(Y=1|X=x')| \leq l_{Y}d(x,x')
\end{equation*}
and
\begin{equation*}
|p(S=1|X=x)-p(S=1|X=x')| \leq l_{S}d(x,x').
\end{equation*}
Then
\begin{equation*}
R_{YS}(\hat{Y}^*_f)-R_{YS}(\hat{Y}^*)\leq (l_{Y}+\lambda l_{S})\mathbb{E}_x[d(x,f(x))].
\end{equation*}
\end{theorem}

Considering the general case without Lipschitz assumptions, the cost of mistrust may take on values in the range $[0,R_{YS}^{\max}-R_{YS}(\hat{Y}^*)]$ as shown in Figure \ref{mistrust_fig}. In example (a), $R_{YS}(\hat{Y}^*_f)-R_{YS}(\hat{Y}^*)=0$, so the minimum cost of mistrust (zero) is achieved. This is achieved if and only if $\hat{Y}^*$ agrees for all points $x$ mapped to the same value $f(x)$, where invertible maps $f$ are a special case. In example (b), $R_{YS}(\hat{Y}^*_f)-R_{YS}(\hat{Y}^*)=R_{YS}^{\max}-R_{YS}(\hat{Y}^*)$ so the maximum cost of mistrust is achieved for a given value of $R_{YS}(\hat{Y}^*)$.  Using the form of $R_{YS}(\hat{Y}_f^*)$ stated in Lemma \ref{R_minimizer}, this is achieved if and only if $P(Y=1|X_f=x)-c_Y\leq\lambda(P(S=1|X_f=x)-c_S)$ and hence $\hat{Y}^*_f(x)=0$ for all supported values $x \in X_f$.

\section{Practical approach}
\label{practical}
\begin{figure*}
\centering
\includegraphics[scale=0.5]{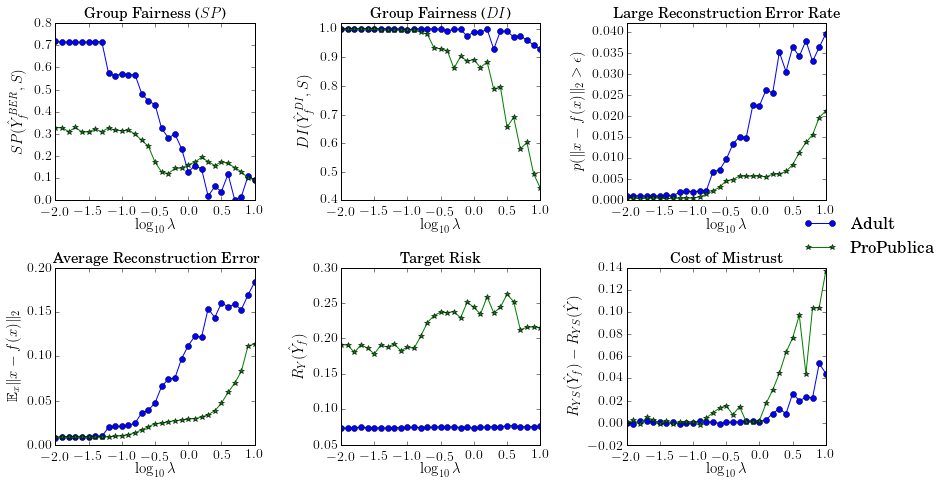}
\caption{Results of experiments on Adult and ProPublica datasets (lower is better on all plots).}
\label{results_fig}
\end{figure*}

Our theoretical results have shown us the desirable properties of a representation function $f$ in terms of both fairness and utility. In practice, we would like to use real-world data to construct such a representation function and estimate the performance measures identified in Section \ref{theory}.

\subsection{Learning a representation function}
\label{learn_representation}
We encode the desirable properties of a representation function $f$ in a cost function $J(f)$, shown in Definition \ref{cost_function}, which we would like to minimize.

\begin{definition}[Cost of representation function $f$]
\label{cost_function}
Let $f$ be function transforming input variable $X$ to $X_f$. Let $d$ be a function measuring reconstruction error. Let $\hat{S}_f$ be an estimate of sensitive variable $S$ using $X_f$ and let 
$J'(\hat{S}_f,S)$
be a sensitive variable cost function. The cost of $f$ is
\begin{equation*}
J(f):=\mathbb{E}_x[d(x,f(x))]-\lambda J'(\hat{S}_f,S)
\end{equation*}
where $\lambda$ is a non-negative constant.
\end{definition}

The reconstruction error term promotes target task utility and individual fairness,\footnote{Another variant would include a $p(d(x,f(x))>\epsilon)$ term to enforce individual fairness. We omit this term for simplicity, since the reconstruction error term should already promote individual fairness.} while the $J'$ term promotes group fairness since we would like to make it difficult to estimate $S$ from $X_f$. The parameter $\lambda$ controls the trade-off between the two terms.

We propose the following specific form of the cost function, which facilitates the optimization of $f$. We use Euclidean distance for the reconstruction error 
\begin{equation*}
d(x,f(x)):=\lVert x-f(x) \rVert_2
\end{equation*}
and cross-entropy for the sensitive variable cost function 
\begin{equation*}
J'(\hat{S}_f,S):=-\mathbb{E}_x[p(S=1|X=x)\log_2 p(\hat{S}_f=1|X=x)+
\end{equation*}
\begin{equation*}
p(S=0|X=x)\log_2 p(\hat{S}_f=0|X=x)].
\end{equation*} 

We estimate and attempt to minimize $J(f)$ over a finite training sample. We estimate $f$ using a fully-connected neural network with one softplus\footnote{The softplus activation function is softplus$(x):=\log_e(1+e^x)$.} hidden layer of 100 units and a linear output layer with the same number of units as the input layer. Hence $f$ is an autoencoder learned using a training objective given by $J$. We estimate $\hat{S}_f$ using a fully-connected neural network with one softplus hidden layer of 100 units and a single sigmoidal output unit. The output layer of $f$, which corresponds to the variable $X_f$, is the input layer for the network estimating $\hat{S}_f$. We alternate updates of the weights in the $f$ and $\hat{S}_f$ networks. The scheme is comparable to \cite{edwards_censoring_2015}, although we do not use the target variable $Y$ when learning $f$ since in general this may not be accessible to the data producer and is not required to learn a representation function with the properties we desire. We optimize the networks using 100 iterations of the Adam Optimizer with a learning rate of 0.0001 and a batch size of 100. We implemented the model in Python using the TensorFlow library. 

\subsection{Learning a decision variable}
\label{learn_decision}
The results presented in Section \ref{theory} show that it is possible to provide performance guarantees for a representation function for \textit{any} target task. In practice, we would like to understand how a decision variable constructed using cleaned input $X_f$ performs on the target task of interest. We are also interested to compare this to a decision variable constructed using the original input $X$.
 
Using Lemma \ref{R_minimizer}, we know the form of the minimizer of the combined risk $R_{YS}$ for group fairness and target utility. If we are only interested in the target variable risk $R_Y$, this equals $R_{YS}$ in the case $\lambda=0$. We produce decision variable $\hat{Y}_f$ by applying the analytic form of $\hat{Y}^*_f$ and separately training networks to estimate $p(Y|X_f)$ and $p(S|X_f)$. We perform this estimation by minimizing the cost functions $J'(\cdot,S)$ and $J'(\cdot,Y)$ respectively over the cleaned training set, in both cases using a fully-connected neural network with one softplus hidden layer of 100 units and a single sigmoidal output unit, and the same training settings previously described. In the same way we produce decision variable $\hat{Y}$ by applying the analytic form of $\hat{Y}^*$ and estimating $p(Y|X)$ and $p(S|X)$ using the original training set.

\subsection{Experiments}
We conducted experiments to test the representation learning approach described in Section \ref{learn_representation} and to illustrate how our theoretical results suggest useful performance measures for the learned representation function $f$. We used two datasets that are well-known in the fair machine learning literature (for example in \cite{calmon2017optimized}), the UCI Adult and ProPublica recividism datasets.\footnote{Available at \url{https://archive.ics.uci.edu/ml/datasets/adult} and \url{https://github.com/propublica/compas-analysis} respectively.} In both datasets we used a training set of 70\% of the data and test on the remaining 30\%. These datasets are summarized in Table \ref{characteristics}, including statistics of interest computed for the training set. 

The Adult dataset contains financial and demographic information compiled from a census about 32561 people and contains 110 input columns once categorical features are binarized. We selected the sensitive variable as gender and the target variable as whether the person's income is at least \$50,000. This setting is similar to a situation where a financial institution makes an algorithmic decision about whether to grant an individual a loan based on a prediction of their income. The training set estimates of $SP(Y,S)$ and $DI(Y,S)$ indicate that there is a clear relationship between the sensitive and target variables.

The ProPublica dataset contains information about 7214 criminal offences committed in Broward County, Florida and contains 79 input columns once categorical features are binarized.\footnote{We processed the free text crime description column by converting it to a categorical variable where descriptions occurring at least 20 times have their own category (covering 82.9\% of all offences) and all other descriptions are marked as `other'. The categorical variable is then binarized.} We selected the sensitive variable as whether the person is of African-American ethnicity and the target variable as whether the person reoffended within two years. This setting is similar to a situation where sentencing decisions are made based on an algorithmic assessment of the individual's likelihood of reoffending, as in the case of the COMPAS tool \cite{compas}. The training set estimates of $SP(Y,S)$ and $DI(Y,S)$ indicate that there is a relationship between the sensitive and target variables, but weaker than in the case of the Adult dataset.

We present the results of our experiments in Figure \ref{results_fig} measured over the test set. We plot a series of performance measures of interest for both datasets, varying the parameter $\lambda$ which controls the extent to which group fairness is considered in the training objective for learning $f$. For all measures, lower is better.

The top row contains measures of group and individual fairness. We compute the statistical parity $SP$ and normalized disparate impact $DI$ using estimates of $\hat{Y}^{BER}_f$ and $\hat{Y}^{DI}_f$ respectively (see Theorems \ref{group_fairness_theorem_simplified} and \ref{DI_minimiser}). For both datasets, these two measures decrease as $\lambda$ increases, which provides evidence that group fairness is improving. We measure the large reconstruction error rate $p(\lVert x - f(x) \rVert_2 > \epsilon)$, which is of interest due to its relationship with individual fairness established in Theorem \ref{individual_fairness_theorem}. A smaller value of this quantity means that greater individual fairness is possible using $X_f$. We select $\epsilon$ to be $\frac{1}{10}$ of the average training set value of $\lVert x \rVert _2$. The results show that there is a trade-off between group and individual fairness, with $p(\lVert x - f(x) \rVert_2 > \epsilon)$ increasing and hence individual fairness becoming worse as $\lambda$ increases.

The bottom row contains measures of target task utility. We show the average reconstruction error, which is related to a particular form of target task risk as shown in Theorem \ref{accuracy_theorem}, which increases with $\lambda$ since the impact of protecting the sensitive variable has a greater influence on $f$. We estimate the target cost-sensitive risk $R_Y$ (see Definition \ref{R_definition}) with $c_Y=\frac{1}{2}$ using the decision variable $\hat{Y}_f$ described in Section \ref{learn_decision}. Interestingly, this risk only slightly increases with $\lambda$ for both datasets, and for the Adult dataset the effect is particularly small. This indicates that removing knowledge of $S$ may only slightly affect target variable prediction performance if other factors in the input data that are not correlated with $S$ are predictive of $Y$. We also estimate the cost of mistrust with parameters $c_S=c_Y=\frac{1}{2}$ (see Definitions \ref{R_definition} and \ref{mistrust_definition}) using the decision variables $\hat{Y}_f$ and $\hat{Y}$ described in Section \ref{learn_decision}, which increases with $\lambda$ as $X_f$ differs more from $X$.

\section{Conclusion and future work}
\label{conclusion}
We proposed learning a representation function to prepare a cleaned version of the input data made available to a decision-making data user. We showed that this approach improves group fairness using common measures such as statistical parity and disparate impact. Furthermore, we showed that it is possible to quantify the cost of such a representation function on individual fairness and target task utility. We also evaluated the \textit{cost of mistrust}, by comparing performance using our approach with a baseline scenario of a trusted data user given access to the original input data and sensitive variable at training time.

We used our results to motivate a practical algorithm that learns representation functions to achieve both fairness and target task utility. Using real-world datasets in financial and criminal justice domains, we found that increasing the fairness parameter $\lambda$ controls the trade-off between group fairness on the one hand, and individual fairness and target task utility on the other. The results highlight the value of the conceptual distinction between group and individual fairness. They also indicate that when the target variable can be predicted using other factors not correlated with the sensitive variable, group fairness can be achieved without too great a cost to target task utility.

Promising future directions include exploring other individual similarity measures besides Euclidean distance and analyzing the ability to extend theoretical guarantees to a finite sample regime. From a regulatory perspective it is interesting to more closely consider cases where the governance model and separation of concerns we propose would be suitable in practice. Protecting the anonymity of user records is closely related to the task of protecting a sensitive variable, indicating the potential to use analysis developed in the context of fairness to privacy applications.

\section{Proofs of theoretical guarantees}
\label{proofs}

We present proofs of the theoretical results stated in Section \ref{theory}.
\vskip -0.3in
\begin{proof}[Proof of Theorem \ref{group_fairness_theorem_simplified}]
$ $

$SP(\hat{Y}_f,S)$

$=p(\hat{Y}=1|S=1)-p(\hat{Y}=1|S=0)$

$=1-p(\hat{Y}=0|S=1)-p(\hat{Y}=1|S=0)$

$= 1-2BER(\hat{Y}_f,S)$

$\leq 1-2BER(\hat{Y}^{BER}_f,S)$

since $\hat{Y}_f^{BER}$ minimizes $BER$ among classifiers receiving $X_f$ as input \cite{zhao_beyond_2013}.
\end{proof}

\begin{proof}[Proof of Theorem \ref{group_fairness_theorem}]
As shown in the proof of Theorem \ref{group_fairness_theorem_simplified}, 

\begin{equation*}
SP(\hat{Y}_f,S)\leq 1-2BER(\hat{Y}^{BER}_f,S).
\end{equation*}

Therefore it is sufficient to show

\begin{equation*}
2BER(\hat{Y}^{BER}_f,S)\geq \frac{H_b^{-1}(H(S|X_f))}{\max(p(S=1),p(S=0))}.
\end{equation*} 

It is helpful to observe that $BER(\hat{Y}^{BER}_f,S)\leq\frac{1}{2}$. We consider the cases (a) $p(S=1)\leq\frac{1}{2}$ and (b) $p(S=1)>\frac{1}{2}$.

(a) First consider the case $p(S=1)\leq\frac{1}{2}$. If in addition 
\begin{equation*}
p(S=0)2BER(\hat{Y}^{BER}_f,S)\geq\frac{1}{2}
\end{equation*}

we have 

$2BER(\hat{Y}^{BER}_f,S)$

$\geq\frac{1}{2p(S=0)}$

$\geq\frac{1}{2\max(p(S=1),p(S=0))}$

$\geq \frac{H_b^{-1}(H(S|X_f))}{\max(p(S=1),p(S=0))}$. 

Otherwise, using Corollary 8 from \cite{zhao_beyond_2013}
\begin{equation*}
H(S|\hat{Y}_f)\leq H_b(p(S=0)2BER(\hat{Y}^{BER}_f,S)).
\end{equation*}

Therefore 

$2BER(\hat{Y}^{BER}_f,S)$

$\geq \frac{H_b^{-1}(H(S|\hat{Y}_f))}{p(S=0)}$

$\geq\frac{H_b^{-1}(H(S|\hat{Y}_f))}{\max(p(S=1),p(S=0))}$ 

where for the first inequality we used the fact that $H_b^{-1}$ is a non-decreasing function.

(b) The case $p(S=1)>\frac{1}{2}$ is similar. If in addition 

\begin{equation*}
p(S=1)2BER(\hat{Y}^{BER}_f,S)\geq\frac{1}{2}
\end{equation*}

we have 

$2BER(\hat{Y}^{BER}_f,S)$

$\geq\frac{1}{2p(S=1)}$

$\geq\frac{1}{2\max(p(S=1),p(S=0))}$

$\geq \frac{H_b^{-1}(H(S|X_f))}{\max(p(S=1),p(S=0))}$. 

Otherwise, again using Corollary 8 from \cite{zhao_beyond_2013}, 

\begin{equation*}
H(S|\hat{Y}_f)\leq H_b(p(S=1)2BER(\hat{Y}^{BER}_f,S)).
\end{equation*} 

Therefore 

$2BER(\hat{Y}^{BER}_f,S)$

$\geq \frac{H_b^{-1}(H(S|\hat{Y}_f))}{p(S=1)}$

$\geq\frac{H_b^{-1}(H(S|\hat{Y}_f))}{\max(p(S=1),p(S=0))}$.

In cases (a) and (b) we have established

\begin{equation*}
2BER(\hat{Y}^{BER}_f,S)\geq\frac{H_b^{-1}(H(S|\hat{Y}_f))}{\max(p(S=1),p(S=0))}.
\end{equation*}

Furthermore, since $\hat{Y}_f$ is a function of $X_f$, $S \to X_f \to \hat{Y}_f$ is a Markov chain \cite{cover2012elements}. Therefore we may apply the Data Processing Inequality \cite{cover2012elements} to yield $H(S|X_f)\leq H(S|\hat{Y}_f)$. Again using the fact that $H_b^{-1}$ is non-decreasing, we have

\begin{equation*}
2BER(\hat{Y}^{BER}_f,S)\geq \frac{H_b^{-1}(H(S|X_f))}{\max(p(S=1),p(S=0))}
\end{equation*}

as required.
\end{proof}

\begin{proof}[Proof of Theorem \ref{DI_minimiser}]
Observe that if $p(\hat{Y}_f=1,S=1)=0$ then $DI$ is not defined. So we assume $p(\hat{Y}_f=1,S=1)>0$. In the steps below we assume $X$ is continuous but in the case $X$ is discrete we may write the same steps replacing integrals with summations.

$DI(\hat{Y}_f,S)$

$=1-\frac{p(\hat{Y}_f=1|S=0)}{p(\hat{Y}_f=1|S=1)}$

by definition

$=1-\frac{\int_xp(X=x|S=0)p(\hat{Y}_f=1|S=0,X=x)dx}{\int_xp(X=x|S=1)p(\hat{Y}_f=1|S=1,X=x)dx}$

law of total expectation

$=1-\frac{\int_xp(X=x|S=0)p(\hat{Y}_f=1|X=x)dx}{\int_xp(X=x|S=1)p(\hat{Y}_f=1|X=x)dx}$

$\hat{Y}_f$ and $S$ are conditionally independent given $X$

$=1-\frac{p(S=1)\int_xp(X=x)p(S=0|X=x)p(\hat{Y}_f=1|X=x)dx}{p(S=0)\int_xp(X=x)p(S=1|X=x)p(\hat{Y}_f=1|X=x)dx}$

Bayes rule

$\leq1-\frac{p(S=1)\int_xp(X=x)(1-\eta_f)p(\hat{Y}_f=1|X=x)dx}{p(S=0)\int_xp(X=x)p(S=1|X=x)p(\hat{Y}_f=1|X=x)dx}$

Since $\eta_f\geq p(S=1|X=x)$, with equality when $\hat{Y}_f=\hat{Y}_f^{DI}$

$\leq1-\frac{p(S=1)\int_xp(X=x)(1-\eta_f)p(\hat{Y}_f=1|X=x)dx}{p(S=0)\int_xp(X=x)\eta_f p(\hat{Y}_f=1|X=x)dx}$

Again since $\eta_f\geq p(S=1|X=x)$, with equality when $\hat{Y}_f=\hat{Y}_f^{DI}$

$=1-\frac{p(S=1)(1-\eta_f)}{p(S=0)\eta_f}.$

Simplifying
\end{proof}

\begin{proof}[Proof of Theorem \ref{individual_fairness_theorem}]
Consider points $x$ and $x'$ drawn independently at random using the input variable $X$.

$p(\hat{Y}_f=1|X=x)$

$=p(\hat{Y}_f=1,X_f=f(x)|X=x)+p(\hat{Y}_f=1,X_f \neq f(x)|X=x)$

law of total probability

$=p(\hat{Y}_f=1,X_f=f(x)|X=x)$

$f$ is deterministic

$=p(X_f=f(x)|X=x)p(\hat{Y_f}=1|X_f=f(x),X=x)$

definition of conditional probability

$=p(\hat{Y_f}=1|X_f=f(x),X=x)$

from definition of $X_f$

$=p(\hat{Y_f}=1|X_f=f(x))$

$\hat{Y}_f$ is conditionally independent of $X$ given $X_f$

$=p(\hat{Y}=1|X=f(x))$.

from definition of $\hat{Y}_f$

Similarly $p(\hat{Y}_f=1|X=x')=p(\hat{Y}=1|X=f(x'))$.

With probability at least $1-\delta$, $d(x,f(x))\leq \epsilon$. Similarly, with probability at least $1-\delta$, $d(x',f(x'))\leq \epsilon$. By the union bound, both statements hold with probability at least $1-2\delta$.

Having established these facts, we prove the result:

$D(p(\hat{Y}_f=1|X=x),p(\hat{Y}_f=1|X=x'))$

$=D(p(\hat{Y}=1|X=f(x)),p(\hat{Y}=1|X=f(x')))$

using the first result established above

$\leq D(p(\hat{Y}=1|X=f(x)),p(\hat{Y}=1|X=x))+$

{\raggedleft $\quad \quad D(p(\hat{Y}=1|X=x),p(\hat{Y}=1|X=f(x')))$}

triangle inequality

$\leq \epsilon+D(p(\hat{Y}=1|X=x),p(\hat{Y}=1|X=f(x')))$

Using the high probability bound above plus the fact that $\hat{Y}$ is individually fair with respect to $X$.

$\leq \epsilon+D(p(\hat{Y}=1|X=x),p(\hat{Y}=1|X=x'))+$

{\raggedleft $\quad \quad D(p(\hat{Y}=1|X=x'),p(\hat{Y}=1|X=f(x')))$}

Using the triangle inequality

$\leq 2\epsilon+d(x,x')$.

Once again using the high probability bound above plus the fact that $\hat{Y}$ is individually fair with respect to $X$.

We have thus established $IU_{D,d_\epsilon}(\hat{Y}_f,X)\leq 2\delta$ as required.
\end{proof}

\begin{proof}[Proof of Theorem \ref{accuracy_theorem}]

$ $

$R_Y(\hat{Y}_f)$

$=\mathbb{E}_x[D(p(\hat{Y}_f=1|X=x),p(Y=1|X=x))]$

$\leq\mathbb{E}_x[D(p(\hat{Y}_f=1|X=x),p(\hat{Y}=1|X=x))+$

{\raggedleft \quad \quad $D(p(\hat{Y}=1|X=x),p(Y=1|X=x))]$}

triangle inequality

$=\mathbb{E}_x[D(p(\hat{Y}=1|X=f(x)),p(\hat{Y}=1|X=x))+$

{\raggedleft \quad \quad $D(p(\hat{Y}=1|X=x),p(Y=1|X=x))]$}

$p(\hat{Y}_f=1|X=x)=p(\hat{Y}=1|X=f(x))$ shown in the proof of Theorem \ref{individual_fairness_theorem}

$=\mathbb{E}_x[D(p(\hat{Y}=1|X=f(x)),p(\hat{Y}=1|X=x))]+R_Y(\hat{Y})$

by linearity of expectations

$\leq\mathbb{E}_x[d(x,f(x))]+R_Y(\hat{Y})$.

since $Y$ is individually fair with respect to $X$
\end{proof}

\begin{proof}[Proof of Theorem \ref{mistrust_theorem}]

$ $

$R_{YS}(\hat{Y}_f^*)- R_{YS}(\hat{Y}^*)$

$\leq R_{YS}(\hat{Y}_f)- R_{YS}(\hat{Y}^*)$

using the definition $\hat{Y}_f(x):=\hat{Y}^*(f(x))$

$=R_{Y}(\hat{Y}_f)- R_{Y}(\hat{Y}^*)-\lambda (R_{S}(\hat{Y}_f)- R_{S}(\hat{Y}^*))$

Definition of $R_{YS}$

$=\mathbb{E}_x[(c_Y-p(Y=1|X=x))(\hat{Y}_f(x)-\hat{Y}^*(x))]-$

{\raggedleft \quad \quad $\lambda\mathbb{E}_x[(c_S-p(S=1|X=x))(\hat{Y}_f(x)-\hat{Y}^*(x))]$}

We may write 
\begin{equation*}
R_{Y}(\cdot)=(1-c)p(Y=1)+\mathbb{E}_x[(c_Y-p(Y=1|X=x))p(\cdot=1|X=x)]
\end{equation*}
and 
\begin{equation*}
R_{S}(\cdot)=(1-c_S)p(S=1)+\mathbb{E}_x[(c_S-p(S=1|X=x))p(\cdot=1|X=x)]
\end{equation*}
(Lemma 10 from \cite{menon2017cost}). Then we apply linearity of expectation.

$=\mathbb{E}_x[(c_Y-p(Y=1|X=x)-$

{\raggedleft \quad \quad $\lambda(c_S-p(S=1|X=x)))(\hat{Y}_f(x)-\hat{Y}^*(x))]$}

Factorizing, linearity of expectation

$\leq\mathbb{E}_x[(l_{Y}+\lambda l_{S})d(x,f(x))]$

See below for discussion

$=(l_{Y}+\lambda l_{S})\mathbb{E}_x[d(x,f(x))]$.

factorizing

Discussion of second last line: consider the minimizer of $R_{YS}$ from Lemma \ref{R_minimizer}. For any $x$ such that $\hat{Y}_f(x)\neq\hat{Y}^*(x)$, there must be some point $x'$ such that 
\begin{equation*}
c_Y-p(Y=1|X=x')-\lambda(c_S-p(S=1|X=x'))=0
\end{equation*}

and $d(x,x')\leq d(x,f(x))$. Hence, applying the Lipschitz conditions we have:

$c_Y-p(Y=1|X=x)-\lambda(c_S-p(S=1|X=x))$

$\leq c_Y-p(Y=1|X=x')+l_{Y}d(x,x')-$

{\raggedleft \quad \quad $\lambda(c_S-p(S=1|X=x')-l_{S}d(x,x'))$}

$ = (l_{Y}+\lambda l_{S})d(x,x')$

$ \leq (l_{Y}+\lambda l_{S})d(x,f(x))$.
\end{proof}

\bibliographystyle{ACM-Reference-Format}
\bibliography{paper} 

\end{document}